\documentclass{amsart}
\title[On the properties of $SL(\alpha)$]{On the properties of $\alpha$-unchaining single linkage hierarchical clustering}
\date{}

\author{A. Martínez-Pérez}

\thanks{The author was partially supported by MTM-2009-07030.}
\usepackage{amsfonts}
\usepackage{amsmath}
\usepackage{latexsym}
\usepackage[all]{xy}
\usepackage{amsthm}
\usepackage{graphicx}
\usepackage[T1]{fontenc}
\usepackage{pifont}
\usepackage{caption}
\usepackage{verbatim}
\usepackage{amsmath}
\usepackage{amssymb}
\usepackage{amsthm}
\usepackage{amscd}
\usepackage{graphics}
\usepackage{graphpap}
\usepackage{float}

\newtheorem{definicion}{Definition}[section]

\newtheorem{prop}[definicion]{Proposition}

\newtheorem{obs}[definicion]{Remark}
\newtheorem{teorema}[definicion]{Theorem}
\newtheorem{cor}[definicion]{Corollary}
\newtheorem{ejp}[definicion]{Example}

\newcommand{\co}{\ensuremath{\colon}} 
\newcommand{\bn}{\ensuremath{\mathbb{N}}} 
\newcommand{\br}{\ensuremath{\mathbb{R}}} 
\newcommand{\cmark}{\ding{51}}
\newcommand{\xmark}{\ding{55}}

\begin{document}

\begin{abstract} In the election of a hierarchical clustering method, theoretic properties may give some insight to determine which method is the most suitable to treat a clustering problem. Herein, we study some basic properties of two hierarchical clustering methods: $\alpha$-unchaining single linkage or $SL(\alpha)$ and a modified version of this one, $SL^*(\alpha)$. We compare the results with the properties satisfied by the classical linkage-based hierarchical clustering methods.
\end{abstract}

\maketitle

\begin{footnotesize}
\end{footnotesize}

\begin{footnotesize}\textit{E-mail}: alvaro.martinezperez@uclm.es
\end{footnotesize}

\begin{footnotesize}
\textit{Address}: Departamento de Análisis Económico y Finanzas. Universidad de Castilla-\newline La Mancha. Avda. Real Fábrica de Seda, s/n. 45600. Talavera de la Reina. 
Toledo.  Spain
\end{footnotesize}

\tableofcontents

\section{Introduction}

Kleinberg discussed in \cite{Kle} the problem of clustering in an axiomatic way. He proposed a few basic properties that any clustering scheme should hold. Let $\mathcal{P}(X)$ denote the set of all possible partitions of $X$. Fix a clustering method $\mathfrak{T}$ so that $\mathfrak{T}(X)=\Pi\in \mathcal{P}(X)$. The properties proposed by Kleinberg were:

\begin{itemize}
	\item Scale invariance: For all $\alpha>0$, $\mathfrak{T}(X,\alpha \cdot d)=\Pi$
	\item Richness: Given a finite set $X$, for every $\Pi\in \mathcal{P}(X)$ there exists a metric $d_\Pi$ on $X$ such that 
	$\mathfrak{T}(X,d_\Pi)=\Pi$.
	\item Consistency: Let $\Pi=\{B_1,...,B_n\}$. Let $d'$ be any metric on $X$ such that
		\begin{itemize}
			\item[1)] for all $x,x'\in B_i$, $d'(x,x')\leq d(x,x')$ and
			\item[2)] for all $x\in B_i$, $x'\in B_j$, $i\neq j$, $d'(x,x')\geq d(x,x')$.
		\end{itemize}
		Then, $\mathfrak{T}(X,d')=\Pi$.
\end{itemize}

Then, he proved that no standard clustering scheme satisfying this conditions simultaneously can exist. This does not mean that defining a clustering function is impossible. The impossibility only holds when the unique input in the algorithm is the space and the set of distances. It can be avoided including, for example, the number of clusters to be obtained as part of the input. See \cite{ABL_10} and \cite{ZB}.

Carlsson and Mémoli studied in \cite{CM} the analogous problem for hierarchical clustering methods taking as input a finite metric space. They set three basic conditions, see Theorem \ref{Th: 18}, and prove that the unique method satisfying these conditions simultaneously is the well-known single linkage algorithm. The authors prove also that single linkage hierarchical clustering 
($SL$ $HC$) exhibits some good properties. In particular, it is stable in the Gromov-Hausdorff sense, this is, if two metric spaces are close in the Gromov-Hausdorff metric, then applying the algorithm, the ultrametric spaces obtained are also close in this metric. However, there is a basic weakness in $SL$ $HC$ which is the \textit{chaining effect} which can be seen as the tendency of the algorithm to merge two blocks when the minimal distance between them is small ignoring everything else in the distribution. 

In \cite{M1} we tried to offer some solution to this effect. We proposed a modified version of $SL$ algorithm, $\alpha$-unchaining single linkage (or $SL(\alpha)$), which shows some sensitivity to the density distribution of the sample and it is capable to distinguish blocks even though the minimal distance between them is small. We also defined a second version of this method, $SL^*(\alpha)$, to detect blocks when they are connected by a chain of points. Then, we studied the unchaining properties of both methods.

Thus, we were able to offer some solution to these chaining effects but, in exchange, we lost some of the good properties of $SL$. In particular, $SL(\alpha)$ is no longer stable in the Gromov-Hausdorff sense. In fact, as we proved in \cite{M2}, there is no stable solution to this chaining effect in the range of almost-standard linkage-based $HC$ methods using $\ell^{SL}$.


Now, the question is when should we use $SL(\alpha)$ and $SL^*(\alpha)$. Among the large variety of clustering methods  the best option usually depends on the particular clustering problem. But how do we choose the most suitable algorithm for the task? Ackerman, Ben David and Loker propose to study significant properties of the clustering functions. See \cite{ABL_10} and \cite{ABL_10a}. The idea is finding abstract significant properties concerning the output of the algorithms which illustrate the difference  between applying one clustering method or another. Then, the practitioner should decide which properties are important for the problem under study and choose the algorithm which satisfies them.

In \cite{M1} and \cite{M2} we proved the chaining, unchaining and stability properties of $SL(\alpha)$ and 
$SL^*(\alpha)$. Herein, we complete the work  by analyzing which of the abstract characteristic properties of $SL$ are also satisfied by these two methods and which properties are lost by adding the unchaining condition. 

We start with the characterization of $SL$ $HC$ by Carlsson and Mémoli. In the original characterization of $SL$ from \cite{CM}, properties $I$, $II$ and $III$ characterize $SL$, see \ref{Th: 18}. However, we introduce some alternative definitions to offer a better picture of the difference. 

We define that a $HC$ method $\mathfrak{T}$ satisfies property $A2$ if adding points to the input will never make increase the distance between previous points in the output. This is the case of $SL$. In fact, we prove that $A2$ together with $A1$ (the algorithm leaves ultrametric spaces invariant) and $A3$ (the distance between two points in the output is at least the minimal $\varepsilon>0$ so that there exists a $\varepsilon$-chain between them in the input), offers an alternative characterization of $SL$. See corollary \ref{Cor: characterization}.

Properties $A1$ and $A3$ are trivially satisfied by many algorithms, in particular $SL(\alpha)$, $SL^*(\alpha)$, complete linkage ($CL$) or average linkage ($AL$). Thus, $A2$ illustrates the difference between $SL$ and other methods as those mentioned above. Also, considering the original characterization from \cite{CM}, it is trivial to check that  $SL(\alpha)$, $SL^*(\alpha)$, $AL$ and $CL$ satisfy $I$ and $III$ but not $II$ and therefore, property $II$ can be used to distinguish those algorithms from $SL$. However, since $II$ implies $A2$, we believe that $A2$ is a better option for the task.

We also prove that other basic properties as being permutation invariant or rich are  satisfied by all of them.





The results obtained  in \cite{M1}, \cite{M2} and herein are summarized in Table \ref{tabla}.

\section{Background and notation}\label{Section: background}

A dendrogram over a finite set is a nested family of partitions. This is usually represented as a rooted tree. 

Let $\mathcal{P}(X)$ denote the collection of all partitions of a finite set $X=\{x_1,...,x_n\}$. Then, a dendrogram can also be described as a map $\theta\co [0,\infty)\to \mathcal{P}(X)$ such that:
\begin{itemize}
	\item[1.] $\theta(0)=\{\{x_1\},\{x_2\},...,\{x_n\}\}$,
	\item[2.] there exists $T$ such that $\theta(t)=X$ for every $t\geq T$,
	\item[3.] if $r\leq s$ then $\theta(r)$ refines $\theta(s)$,
	\item[4.] for all $r$ there exists $\varepsilon >0$ such that $\theta(r)=\theta(t)$ for $t\in [r,r+\varepsilon]$.
\end{itemize}

Notice that conditions 2 and 4 imply that there exist $t_0<t_1<...<t_m$ such that $\theta(r)=\theta(t_{i-1})$ for every $r\in [t_{i-1},t_i)$, $i=0,1,...,m$ and $\theta(r)=\theta(t_{m})=\{X\}$ for every $r\in [t_m,\infty)$.

For any partition $\{B_1,...,B_k\}\in \mathcal{P}(X)$, the subsets $B_i$ are called \textit{blocks}.

Let $\mathcal{D}(X)$ denote the collection of all possible dendrograms over a finite set $X$. Given some $\theta \in \mathcal{D}(X)$, let us denote $\theta(t)=\{B_1^t,...,B_{k(t)}^t\}$. Therefore, the nested family of partitions is given by the corresponding partitions at $t_0,...,t_m$, this is, $\{B_1^{t_i},...,B_{k(t_i)}^{t_i}\}$, $i=0,...,m$.

An \emph{ultrametric space} is a metric space $(X,d)$ such that 
$d(x,y)\leq \max \{d(x,z),d(z,y)\}$
for all $x,y,z\in X$. Given a finite metric space $X$ let $\mathcal{U}(X)$ denote the set of all ultrametrics over $X$.

There is a well known equivalence between trees and ultrametrics. See \cite{Hug} and \cite{M-M} for a complete exposition of how to build categorical equivalences between them. In particular, this may be translated into an equivalence between dendrograms and ultrametrics:

Thus, a hierarchical clustering method $\mathfrak{T}$ can be presented as an algorithm whose output is a dendrogram or an ultrametric space.  Let  $\mathfrak{T}_{\mathcal{D}}(X,d)$ denote the dendrogram obtained by applying $\mathfrak{T}$ to a metric space $(X,d)$ and $\mathfrak{T}_{\mathcal{U}}(X,d)$ denote the corresponding ultrametric space.

In \cite{CM} the authors use a recursive procedure to redefine $SL$ $HC$, average linkage ($AL$) and complete linkage ($CL$) hierarchical clustering. The main advantage of this procedure is that it allows to merge more than two clusters at the same time. Therefore,  $AL$ and $CL$ $HC$ can be made \textit{permutation invariant}, meaning that the result of the hierarchical clustering does not depend on the order in which the points are introduced in the algorithm. In \cite{M1} we gave an alternative presentation of this recursive procedure as a first step to define $SL(\alpha)$ and $SL^*(\alpha)$. Let us recall here, for completeness, this presentation.


For $x, y \, \in  \, X$ and any (standard) clustering $C$ of $X$, $x \sim_C y$ if $x$ and $y$ belong to the same cluster in $C$ and $x \not \sim_C y$, otherwise.

Two (standard) clusterings  $C = (C_1,...,C_k)$ of $(X, d)$ and $C' = (C'_1,...C'_k)$ 
of $(X',d')$ are isomorphic clusterings, denoted $(C,d) \cong (C',d')$, if there exists a bijection
$\phi : X \to X'$ such that for all $x, y \, \in  \, X$, $d(x, y) = d'(\phi(x),\phi(y))$ and $x \sim_C y$ if and
only if $\phi(x) \sim_{C'} \phi(y)$.

\begin{definicion} A \textit{linkage function} is a function 
\[\ell : \{(X_1,X_2, d) \, | \, d \mbox{ is a distance function over } X_1 \cup X_2 \} \to R^+ \]
such that,
\begin{itemize}
\item[1.] $\ell$ is \textit{representation independent}: For all $(X_1,X_2)$ and $(X'_1,X'_2)$, if $(X_1,X_2, d) \cong (X'_1,X'_2, d')$ (i.e.,
they are clustering-isomorphic), then $\ell(X_1,X_2, d) = \ell(X'_1,X'_2, d')$.
\item[2.] $\ell$ is \textit{monotonic}: For all $(X_1,X_2)$ if $d'$ is a distance function over $X_1\cup X_2$ such that for all $x \sim_{\{X_1,X_2\}}y$, $d(x, y) = d'(x, y)$ and for all $x \not \sim_{\{X_1,X_2\}}y$, $d(x, y) \leq d'(x, y)$ then $\ell(X_1,X_2, d')\geq \ell(X_1,X_2, d) $.
\item[3.] Any pair of clusters can be made arbitrarily distant: For any pair of data sets $(X_1, d_1)$, $(X_2, d_2)$, and
any $r$ in the range of $\ell$, there exists a distance function $d$ that extends $d_1$ and $d_2$ such that $\ell(X_1,X_2, d) >r$.
\end{itemize}
\end{definicion}

For technical reasons, it is usually assumed that a linkage function has a countable range. Say, the set of nonnegative
algebraic real numbers.

Some standard choices for $\ell$ are:

\begin{itemize}
	\item Single linkage: $\ell^{SL}(B,B')=\min_{(x,x')\in B\times B'}d(x,x')$ 
	\item Complete linkage: $\ell^{CL}(B,B')=\max_{(x,x')\in B\times B'}d(x,x')$
	\item Average linkage: $\ell^{AL}(B,B')=\frac{\sum_{(x,x')\in B\times B'}d(x,x')}{\#(B)\cdot \#(B')}$ where $\#(X)$ denotes the cardinality of the set $X$.
\end{itemize}

Let $(X,d)$ be a finite metric space where $X=\{x_1,...,x_n\}$. Let $L$ denote a family of linkage functions on $X$ and fix some linkage function $\ell\in L$. Then, let
$\mathfrak{T}_\mathcal{D}(X,d)=\theta^\ell$ be as follows:
 
\begin{itemize}
	\item[1.] Let $\Theta_0:=\{x_1,...,x_n\}$ and $R_0=0$.
	\item[2.] For every $i\geq 1$, while $\Theta_{i-1}\neq \{X\}$, let $R_{i}:=\min\{\ell(B,B')\, | \, B,B'\in \Theta_{i-1}, \ B\neq B'\}$. Then, let $G_{R_{i}}^\ell$ be a graph whose vertices are the blocks of $\Theta_{i-1}$ and such that there is an edge joining $B$ and $B'$ if and only if $\ell(B,B')\leq R_{i}$. 
	\item[3.] Consider the equivalence relation $B\sim_{\ell,R} B'$ if and only if $B,B'$ are in the same connected component of $G_R^\ell$. Then, $\Theta_{i}=\frac{\Theta_{i-1}}{\sim_{\ell,R_{i}}}$.
	\item[4.] Finally, let $\theta^\ell\colon [0,\infty)\to \mathcal{P}(X)$ be such that $\theta^\ell(r):=\Theta_{i(r)}$ with $i(r):=\max\{i\, | \, R_i\leq r\}$.
\end{itemize} 

In \cite{M2} the methods defined by applying this algorithm for some linkage function $\ell$ are called \emph{standard linkage-based} $HC$ methods.

Let us now recall the definition of $SL(\alpha)$ and $SL^*(\alpha)$. Further explanations, figures and easy examples of applications of these methods can be found in \cite{M1}.

Given a finite metric space $(X,d)$, let $F_t(X,d)$ be the Rips (or Vietoris-Rips) complex of $(X,d)$. Let us recall that the Rips complex of a metric space $(X,d)$ is a simplicial complex whose vertices are the points of $X$ and $[v_0,...,v_k]$ is a simplex of $F_t(X,d)$ if and only if $d(v_i,v_j)\leq t$ for every $i,j$. Given any subset $Y\subset X$, by $F_t(Y)$ we refer to the subcomplex of $F_t(X)$ defined by the vertices in $Y$. A simplex $[v_0,...,v_k]$ has dimension $k$. The dimension of a simplicial complex is the maximal dimension of its simplices.

Let $X=\{x_1,...,x_n\}$. Let $d_{ij}:=d(x_i,x_j)$ and  $D:=\{t_i \, : \, 0\leq i \leq m\}=\{d_{ij} \, : \, 1\leq i, j \leq n\}$ with $t_i<t_j$ $\forall \, i<j$ where ``<'' denotes the order of the real numbers. Clearly, $t_0=0$.


Let the dendrogram defined by $SL(\alpha)$, $\mathfrak{T}^{SL(\alpha)}_\mathcal{D}(X,d)=\theta_{X,\alpha}$ or simply $\theta_\alpha$, be as follows:

\begin{itemize}

	\item[1)] Let $\theta_\alpha(0):=\{\{x_1\},...,\{x_n\}\}$ and $\theta_\alpha(t):=\theta_\alpha(0)$ $\forall t< t_1$. Now, for every $i$, given $\theta_\alpha[t_{i-1},t_i)=\theta_\alpha(t_{i-1})=\{B_1,...,B_m\}$, we define recursively  $\theta_\alpha$ on the interval $[t_i,t_{i+1})$ as follows: 

	\item[2)] Let $G_\alpha^{t_i}$ be a graph with vertices $\mathcal{V}(G_\alpha^{t_i}):=\{B_1,...,B_m\}$ and edges $\mathcal{E}(G_\alpha^{t_i}):=\{B_j,B_k\}$ such that the following conditions hold:
		\begin{itemize}
			\item[i)] $\min\{d(x,y)\, | \, x\in B_j,\ y\in B_k\}\leq t_i$.
			\item[ii)] there is a simplex $\Delta \in F_{t_i}(B_j\cup  B_k)$  such that $\Delta \cap B_j\neq \emptyset$, $\Delta \cap B_k\neq \emptyset$ and $\alpha \cdot dim(\Delta)\geq \min\{dim (F_{t_i}(B_j)), dim (F_{t_i}(B_k))\}$. 
		\end{itemize}


	\item[3)] Let us define a relation, $\sim_{t_i,\alpha}$ as follows.
		
Let $B_j\sim_{t_i,\alpha}B_k$ if $B_j,B_k$ belong to the same connected component of the graph $G_\alpha^{t_i}$. Then, $\sim_{t_i,\alpha}$ induces an equivalence relation.

	\item[4)] For every  $t\in [t_i,t_{i+1})$, $\theta_\alpha(t):=\theta_\alpha(t_{i-1})/\sim_{t_i,\alpha}$.
\end{itemize}

This construction is generalized in \cite{M2} to define the class of \emph{almost-standard linkage-based} $HC$ methods.

\begin{obs}\label{Remark: necessary chain} Notice that if two points $x,x'$ belong to the same block  of $\theta_{\alpha}(t_i)$  then, necessarily, there exists a $t_i$-chain, $x=x_0,x_1,...,x_n=x'$ joining them. In particular, if $x_j\in B_j\in \theta_\alpha(t_{i-1})$, $j=0,...,n$, the corresponding edges $\{B_{j-1},B_j\}$, $1\leq j \leq n$, satisfy condition $ii$. This is immediate by construction. 
\end{obs}

Let the dendrogram defined by $SL^*(\alpha)$, $\mathfrak{T}^{SL^*(\alpha)}_\mathcal{D}(X,d)=\theta^*_{X,\alpha}$ or simply $\theta^*_\alpha$,  be as follows:


\begin{itemize}
	\item[1)] Let $\theta^*_\alpha(0):=\{\{x_1\},...,\{x_n\}\}$ and $\theta^*_\alpha(t):=\theta^*_\alpha(0)$ $\forall t< t_1$. 
	
	Now, given $\theta^*_\alpha[t_{i-1},t_i)=\theta^*(t_{i-1})=\{B_1,...,B_m\}$, we define recursively  $\theta^*_\alpha$ on the interval $[t_i,t_{i+1})$ as follows:

	\item[2)] Let $G_\alpha^{t_i}$ be a graph with vertices $\mathcal{V}(G_\alpha^{t_i}):=\{B_1,...,B_m\}$ and edges $\mathcal{E}(G_\alpha^{t_i}):=\{B_j,B_k\}$ such that the following conditions hold:
		\begin{itemize}
			\item[i)] $\min\{d(x,y)\, | \, x\in B_j,\ y\in B_k\}\leq t_i$.
			\item[ii)] there is a simplex $\Delta \in F_{t_i}(B_j\cup  B_k)$  such that $\Delta \cap B_j\neq \emptyset$, $\Delta \cap B_k\neq \emptyset$ and $\alpha \cdot dim(\Delta)\geq \min\{dim (F_{t_i}(B_j)), dim (F_{t_i}(B_k))\}$. 
		\end{itemize}

By an abuse of the notation, we may write $B$ to refer both to the block of $\theta(t_{i-1})$ and to the vertex of $G_\alpha^{t_i}$.

	\item[3)] Let us define a relation, $\sim_{t_i,\alpha}$ between the blocks as follows.

Let $cc(G_\alpha^{t_i})$ be the set of connected components of the graph $G_\alpha^{t_i}$. Let $A\in cc(G_\alpha^{t_i})$ with $A=\{B_{j_1},...,B_{j_r}\}$.

Let us call \emph{big blocks} of $A$ those blocks such that
\begin{equation}\label{Big blocks} \alpha \cdot \#(B_{j_k})\geq \max_{1\leq l \leq r}\{ \#(B_{j_l})\}.\end{equation}

The rest of blocks of $A$ are called \emph{small blocks}.

Let $H_\alpha(A)$ be the subgraph of $A$ whose vertices are the big blocks and $S_\alpha(A)$ be the subgraph of $A$ whose vertices are the small blocks.
	
Then, $B_{j_k}\sim_{t_i,\alpha} B_{j_{k'}}$ if one of the following conditions holds:

		\begin{itemize}
			\item[iii)] $\exists \, C\in cc(H_\alpha(A))$ such that $B_{j_k},B_{j_{k'}}\in C$. 
			\item[iv)] $B_{j_k}\in C \in cc(H_\alpha(A))$, $B_{j_{k'}}\in C'\in  cc(S_\alpha(A))$ and there is no big block in $A\backslash C$ adjacent to any block in $C'$.

		\end{itemize}

Then, $\sim_{t_i,\alpha}$ induces an equivalence relation whose classes are contained in the connected components of $G_\alpha^{t_i}$.

	\item[4)] For every  $t\in [t_i,t_{i+1})$, $\theta^*_\alpha(t):=\theta^*_\alpha(t_{i-1})/\sim_{t_i,\alpha}$.
\end{itemize}

\begin{obs}\label{Remark: connected} At step $iii$, if $H_{\alpha}(A)$ is connected, then $B_{j_1}\cup \cdots \cup B_{j_r}$ defines a block of $\theta_\alpha(t_i)$. 
\end{obs}

\begin{obs} \label{Remark: necessary chain 2} Notice that Remark \ref{Remark: necessary chain} still applies. In fact, if two points $x,x'$ belong to the same block  of $\theta^*_{\alpha}(t_i)$  then, necessarily, there exists a $t_i$-chain, $x=x_0,x_1,...,x_n=x'$ joining them so that if $x_j\in B_j\in \theta^*_\alpha(t_{i-1})$, $j=0,...,n$, the corresponding edges $\{B_{j-1},B_j\}$, 
$1\leq j \leq n$, satisfy condition $ii$.  
\end{obs}

\section{Single linkage hierarchical clustering}\label{Section: characterization}

In this section we recall some basic properties and the characterization of $SL$ $HC$ from \cite{CM}. We also propose some alternatives. Our first intention is to find significant properties to compare $SL$ and $SL(\alpha)$.


\subsection{Characterization of $SL$}

Carlsson and M\'{e}moli provided the following axiomatic characterization of $SL$ $HC$:

Let us recall that given a finite metric space $(X,d)$, $sep(X,d):=min_{x\neq x'}d(x,x')$.

\begin{teorema}\cite[Theorem 18]{CM}\label{Th: 18} Let $\mathfrak{T}$ be a hierarchical clustering method such that:

\begin{itemize}
	\item[(I)] $\mathfrak{T}_\mathcal{U}\Big(\{p,q\}, \left( 
\begin{array}{ccc}
\delta &  0 \\ 
0 & \delta \end{array}
\right)\Big)=
\Big(\{p,q\}, \left( 
\begin{array}{ccc}
\delta &  0 \\ 
0 & \delta \end{array}
\right)\Big)$ for all $\delta >0$.

	\item[(II)] Given two finite metric spaces $X,Y$ and $\phi\colon X\to Y$ such that $d_X(x,x')\geq d_Y(\phi(x),\phi(x'))$ for all $x,x'\in X$, then 
					\[u_X(x,x')\geq u_Y(\phi(x),\phi(x'))\]
	also holds for all $x,x'\in X$, where $\mathfrak{T}_{\mathcal{U}}(X,d_X)=(X,u_X)$ and $\mathfrak{T}_{\mathcal{U}}(Y,d_Y)=(Y,u_Y)$.

	\item[(III)] For any metric space $(X,d)$, \[u(x,x')\geq sep(X,d) \mbox{ for all } x\neq x'\in X\] where $\mathfrak{T}_{\mathcal{U}}(X,d)=(X,u)$. 
\end{itemize}

Then, 
$\mathfrak{T}$ is exactly \textit{single linkage hierarchical clustering}.
\end{teorema}

\textbf{Notation}: For the particular case of $SL$ $HC$, if there is no need to distinguish the metric space, let us denote $\mathfrak{T}^{SL}_{\mathcal{D}}(X,d)=\theta_{SL}$ and $\mathfrak{T}^{SL}_{\mathcal{U}}(X,d)=(X,u_{SL})$. 

\textbf{Notation}: Given two metrics $d,d'$ defined on a set $X$, let us denote $d\leq d'$ if $d(x,x')\leq d'(x,x')$ $\forall \, x,x'\in X$.

The following propositions follow immediately from the proof of \cite[Theorem 18]{CM}.

\begin{prop}\label{Prop: at least} For any metric space $(X,d)$, if $\mathfrak{T}$ satisfies conditions $II$ and $III$, then  $u \geq u_{SL}$.
\end{prop}  

This is, the ultrametric distance between two points is at least the minimal length $\varepsilon$ for which there is a $\varepsilon$-chain joining them.

It is readily seen that if $u \geq u_{SL}$, then $\mathfrak{T}$ satifies $III$.

\begin{prop}\label{Prop: at most} If $\mathfrak{T}$ satisfies conditions $I$ and $II$, then  $u_{SL}\geq u$.
\end{prop}

In fact, Proposition \ref{Prop: at most} can be improved introducing the following condition.

\begin{itemize}
\item[A2)]  Let $(Y,d)$ be a metric space and $X\subset Y$. If $i\colon X \to Y$ is the inclusion map, then $u_X(x,x')\geq u_Y(i(x),i(x'))$. 
\end{itemize}

This is, by adding points to the space we may make the ultrametric distance smaller but never bigger. Clearly, $II \Rightarrow A2$. The proof of Proposition \ref{Prop: at most}, \cite{CM},  can be trivially adapted to obtain the following.

\begin{prop}\label{Prop: at most 2} If $\mathfrak{T}$ satisfies conditions $I$ and $A2$, then  $u_{SL}\geq u$.
\end{prop}  

\begin{proof} Let $x,x'\in (X,d)$ such that $u_{SL}(x,x')=\delta$. Then, there exists a $\delta$-chain $x=x_0,x_1,...,x_n=x'$ such that $\max_id(x_{i-1},x_i)=\delta$. By $I$, if $X_i=\{x_{i-1},x_i\}$, $\mathfrak{T}(X_i,d|_{X_i})=(X_i,d|_{X_i})$ and $u_{X_i}(x_{i-1},x_i)\leq \delta$. Then, by $A2$, $u(x_{i-1},x_i)\leq u_{X_i}(x_{i-1},x_i)\leq \delta$  and, by the properties of the ultrametric, $u(x,x')\leq \delta$. 
\end{proof}

Another natural condition to ask on a hierarchical clustering method is leaving invariant any ultrametric space:

\begin{itemize}
\item[A1)]  If $(X,d)$ is an ultrametric space, then $u(x,y)=d(x,y)$. 
\end{itemize}

This is, applying the hierarchical clustering method to an ultrametric space we obtain the same ultrametric space.

Also, it can be readily seen that $SL$ $HC$ satisfies $A1$: 

\begin{prop} \label{Prop: ultram_0} If $(X,d)$ is an ultrametric space, then $u_{SL}(x,y)=d(x,y)$ for every $x,y\in X$.
\end{prop}

\begin{proof} By definition, it is clear that $u_{SL}(x,y)\leq d(x,y)$ for every $x,y\in X$.

Let us see that, if $(X,d)$ is an ultrametric space, then $u_{SL}(x,y)\geq d(x,y)$. $u_{SL}(x,y)=\inf\{t \, | \, \mbox{there exists a $t$-chain joining } x \mbox{ to } y\}$. Suppose $u_{SL}(x,y)=t$ and let $x=x_0,x_1,...,x_n=y$ be a $t$-chain joining $x$ to $y$. By the properties of the ultrametric, $d(x_{i-1},x_{i+1})\leq \max\{d(x_{i-1},x_{i}),d(x_{i},x_{i+1})\}\leq t$ for every $1\leq i\leq n-1$. Therefore, $d(x,y)\leq t$ and $u_{SL}(x,y)\geq d(x,y)$.
\end{proof}

Richness property for $HC$ methods can be defined in the same way Kleinberg did for standard clustering. Thus, a $HC$ method $\mathfrak{T}$ satisfies \emph{richness property} if given a finite set $X$, for every $\theta \in \mathcal{D}(X)$ there exists a metric $d_\theta$ on $X$ such that $\mathfrak{T}_{\mathcal{D}}(X,d_\theta)=\theta$.

\begin{cor}\label{Cor: rich_1} $\mathfrak{T}^{SL}$ satisfies richness property.  
\end{cor}

It is trivial to check that $A1  \Rightarrow I$ (and $A3  \Rightarrow III$). Therefore, by Proposition \ref{Prop: at most 2}, we obtain also the following alternative characterization of $SL$ $HC$.

\begin{cor}\label{Cor: characterization} Let $\mathfrak{T}$ be a hierarchical clustering method such that:
\begin{itemize}
\item[A1)] If $(X,d)$ is an ultrametric space, then $u_X(x,y)=d(x,y)$.
\item[A2)] Let $(Y,d)$ be a metric space and $X\subset Y$. If $i\colon X \to Y$ is the inclusion map, 
then $u_X(x,x')\geq u_Y(i(x),i(x'))$.
\item[A3)] $u\geq u_{SL}$.
\end{itemize}

Then, $\mathfrak{T}$ is exactly SL HC.
\end{cor}

\subsection{Stability of $SL$}

Let us recall the definition of Gromov-Hausdorff distance from \cite{BBI}. See also \cite{Gr}.

Let $(X,d_X)$ and $(Y,d_Y)$ two metric spaces. A correspondence (between $A$ and $B$) is a subset $R\in A\times B$ such that
\begin{itemize}
	\item $\forall \, a\in A$, there exists $b\in B$ s.t. $(a,b)\in R$
	\item $\forall \, b\in B$, there exists $a\in A$ s.t. $(a,b)\in R$
\end{itemize}

Let $\mathcal{R}(A,B)$ denote the set of all possible correspondences between $A$ and $B$.

Let $\Gamma_{X,Y}\colon X\times Y \times X \times Y\to \br^+$ given by \[(x,y,x',y')\mapsto |d_X(x,x')-d_Y(y,y')|.\]

Then, the \textit{Gromov-Hausdorff distance} between $X$ and $Y$  is:

\[d_{\mathcal{GH}}(X,Y):=\frac{1}{2} \inf_{R\in \mathcal{R}(X,Y)} \sup_{(x,y)(x',y')\in R}\Gamma_{X,Y}(x,y,x',y').\]

The \textit{Gromov-Hausdorff metric} gives a notion of distance between metric spaces. One of the advantages of this metric is that it is well defined for metric spaces of different cardinality. 
In \cite{CM} this metric is used to prove that $\mathfrak{T}^{SL}$ holds some stability under small perturbations on the metric. The authors prove that if two metric spaces are close (in the Gromov-Hausdorff metric) then the corresponding ultrametric spaces obtained as output of the algorithm are also close. In \cite{M2} we studied Gromov-Hausdorff stability of linkage-based $HC$ methods defining the following conditions.

\textbf{Notation:} Let $(\mathcal{M},d_{GH})$ denote the set of finite metric spaces with the Gromov-Hausdorff metric and $(\mathcal{U},d_{GH})$ denote the set of finite ultrametric spaces with the Gromov-Hausdorff metric. 

\begin{definicion} A  $HC$ method $\mathfrak{T}$ is \emph{semi-stable in the Gromov-Hausdorff sense}  if for any sequence of finite metric spaces $((X_k,d_k))_{k\in \bn}$ in $(\mathcal{M},d_{GH})$ such that $\lim_{k\to \infty}(X_k,d_k)=(U,d)\in \mathcal{U}$ then $\lim_{k\to \infty}\mathfrak{T}_\mathcal{U}(X_k,d_k)=  \mathfrak{T}_{\mathcal{U}}(U,d)$.
\end{definicion}

\begin{definicion} A $HC$ method $\mathfrak{T}$ is \emph{stable in the Gromov-Hausdorff sense}  if 
\[\mathfrak{T}_\mathcal{U}\co (\mathcal{M},d_{GH})\to (\mathcal{U},d_{GH})\] is continuous.
\end{definicion}



A hierarchical clustering method is said to be \textit{permutation invariant} if it yields the same dendrogram under permutation of the points in the sample this is, if  the output of the algorithm does not depend on the order by which the data is introduced. Although this is not the easiest way to check this property, it may be noticed that being stable in the Gromov-Hausdorff sense implies being permutation invariant.

The following result is a consequence of \cite[Proposition 26]{CM}. 

\begin{prop} $SL$ $HC$ is stable in the Gromov-Hausdorff sense. In particular, it is semi-stable and permutation invariant. 
\end{prop}

\section{Basic properties of $SL(\alpha)$ and $SL^*(\alpha)$}\label{Section: Basic properties}

In this section, we study some basic properties on $SL(\alpha)$ and $SL^*(\alpha)$. In particular, we check those seen at 
Section \ref{Section: characterization} . 

The following result is clear from the definition.

\begin{prop} $SL(\alpha)$ and $SL^*(\alpha)$ are permutation invariant algorithms.
\end{prop}

\begin{prop}\label{Prop: n/2} Let $(X,d)$ be a finite metric space with $X=\{x_1,...,x_n\}$. If $\alpha\geq \frac{n-2}{2}$, then 
$\mathfrak{T}^{SL}(X)=\mathfrak{T}^{SL(\alpha)}(X)$.
\end{prop}

\begin{proof} Let $\mathfrak{T}^{SL}_\mathcal{D}(X)=\theta_{SL}$, $\mathfrak{T}^{SL(\alpha)}_\mathcal{D}(X)=\theta_{\alpha}$.

We know that $\theta_{SL}(t_{0})=\theta_\alpha(t_{0})$. Suppose $\theta_{SL}(t_{i-1})=\theta_\alpha(t_{i-1})$.

Let us see that for $\alpha\geq \frac{n-2}{2}$, condition $i$ already implies $ii$ and the edges of the graph $G_\alpha^{t_i}$ are those defined by condition $i$. Let $B_1,B_2$ two blocks in $\theta_\alpha(t_{i-1})$ such that $\min\{d(x,y)\, | \, x\in B_1,\ y\in B_2\}\leq t_i$. For any simplex $\Delta$, $\alpha \cdot dim(\Delta)\geq \alpha$ and $\min\{\#(B_1),\#(B_2)\}\leq \frac{n}{2}$. Since $\alpha\geq \frac{n-2}{2}$, $\alpha \cdot dim(\Delta)\geq \alpha\geq \min\{\#(B_1)-1,\#(B_2)-1\}\geq \min\{dim (F_{t_i}(B_1)),dim (F_{t_i}(B_2))\}$.

Then, $\theta_{SL}(t_{i})=\theta_\alpha(t_{i})$.
\end{proof}

\begin{prop}\label{Prop: n-1} Let $(X,d)$ be a finite metric space with $X=\{x_1,...,x_n\}$. If $\alpha\geq n-1$, 
then $\mathfrak{T}^{SL}(X)=\mathfrak{T}^{SL^*(\alpha)}(X)$.
\end{prop}

\begin{proof} Let $\mathfrak{T}^{SL}_\mathcal{D}(X)=\theta_{SL}$ and $\mathfrak{T}^{SL^*(\alpha)}_\mathcal{D}(X)=\theta^*_{\alpha}$.

We know that $\theta_{SL}(t_{0})=\theta^*_{\alpha}(t_0)$. Suppose $\theta_{SL}(t_{i-1})=\theta^*_\alpha(t_{i-1})$.

As we saw in the proof of Proposition \ref{Prop: n/2}, since $\alpha\geq n-1 > \frac{n-2}{2}$, condition $i$ already implies $ii$ and the edges of the graph $G_\alpha^{t_i}$ are those defined by condition $i$. 

Now, let $A=\{B_1,...,B_r\}$ be any connected component of $G_\alpha^{t_i}$. 

If the subgraph $H_\alpha(A)$ is not connected, then there are at least three blocks $B_{i_1}$, $B_{i_2}$, $B_{i_3}$ in $A$, such that $1\leq \#(B_{i_1})<\frac{1}{\alpha}\max_{1\leq l \leq r}\{ \#(B_l)\}$ and $\#(B_{i_2}),\#(B_{i_3})\geq \frac{1}{\alpha}\max_{1\leq l \leq r}\{ \#(B_l)\}$. Trivially, $\max_{1\leq l \leq r}\{ \#(B_l)\}\leq n-2$. Hence, there is a contradiction since $1\leq \frac{n-2}{ \alpha}\leq \frac{n-2}{n-1}<1$. 

Thus, $H_\alpha(A)$ is connected and, as we saw in Remark \ref{Remark: connected}, all the blocks in $A$ are identified. Therefore, $\theta^*_\alpha(t_i)=\theta_{SL}(t_i)$.
\end{proof}

\textbf{Notation:} Let $X$ be a finite metric space. Let us recall that if there is no ambiguity on the metric space we denote $\mathfrak{T}^{SL}_\mathcal{D}(X)=\theta_{SL}$,  $\eta(\theta_{SL})=u_{SL}$ and $\mathfrak{T}^{SL(\alpha)}_\mathcal{D}(X)=\theta_{\alpha}$. Let us denote $\eta(\theta_\alpha)=u_\alpha$.  Similarly, let $\mathfrak{T}^{SL^*(\alpha)}_\mathcal{D}(X)=\theta^*_\alpha$ and $\eta(\theta^*_\alpha)=u^*_\alpha$.

\begin{prop}\label{Prop: mayor} $u_{SL} \leq u_{\alpha}$ and $u_{SL} \leq u^*_\alpha$ for every $\alpha \in \bn$ (i.e. $\mathfrak{T}^{SL(\alpha)}$ and $\mathfrak{T}^{SL^*(\alpha)}$ satisfy $A3$). 
\end{prop}

\begin{proof} As we saw at remarks \ref{Remark: necessary chain} and \ref{Remark: necessary chain 2}, if two points $x,x'\in X$ belong to the same block of $\theta_\alpha(t)$ (resp. $\theta^*_\alpha(t)$), they belong, in particular, to the same $t$-component of $X$ and, therefore, to the same block of $\theta_{SL}(t)$. Thus, $u_{SL}(x,x')\leq u_{\alpha}(x,x')$ (resp. $u_{SL}(x,x')\leq u^*_{\alpha}(x,x')$).
\end{proof}


\begin{prop}\label{Prop: ultram_1} If $(X,d)$ is an ultrametric space, then $\theta_{\alpha}=\theta_{SL}=\theta^*_{\alpha}$ for every $\alpha$.
\end{prop}

\begin{proof} By definition, $\theta_\alpha(t_{0})=\theta_{SL}(t_{0})=\theta^*_{\alpha}(t_{0})$. Suppose $\theta_\alpha(t_{i-1})=\theta_{SL}(t_{i-1})=\theta^*_{\alpha}(t_{i-1})=\{B_1,...,B_n\}$. Let us see that $\theta_\alpha(t_{i})=\theta_{SL}(t_{i})=\theta^*_{\alpha}(t_{i})$.

Let $B_i$, $B_j$ be such that $\min\{d(x,y)\, | \, x\in B_i,\ y\in B_j\}\leq t_i$. Since $B_i$, $B_j$ are $(t_{i-1})$-components, by the properties of the ultrametric, 
$d(x,y)\leq t_{i}$ for every $(x,y)\in B_1\times B_2$.

Therefore, every pair of points in $B_1\cup B_2$ define a simplex in $F_{t_i}(B_1\cup B_2)$ and condition $ii$ holds for every $\alpha$. Thus, there is an edge defined between $B_i$ and $B_j$. This proves that $\theta_{\alpha}=\theta_{SL}$.

Now, let  $B_i,B_j$ be two blocks in the same connected component of $G_\alpha^{t_i}$. Then, by the properties of the ultrametric, $\{B_i,B_j\}$ is an edge of $G_\alpha^{t_i}$. Hence, $H_\alpha(G_\alpha^{t_i})$ is connected and, as we saw in Remark \ref{Remark: connected}, $\theta^*_\alpha(t_i)$ is defined by the connected components of $G_\alpha^{t_i}$. This proves that $\theta^*_{\alpha}=\theta_{SL}$. 
\end{proof}

\begin{cor} \label{Prop: ultram_2} If $(X,d)$ is an ultrametric space, then $u_{\alpha}(x,y)=u^*_\alpha(x,y)=d(x,y)$ for every $x,y\in X$.
\end{cor}

\begin{cor} $SL(\alpha)$ and $SL^*(\alpha)$ satisfy $A1$ and $A3$ but not $A2$. 
\end{cor}

Notice that if $A2$ were also satisfied then, by Corollary \ref{Cor: characterization}, the method would be exactly $SL$. For an example of how these methods fail to satisfy $A2$ consider the following example from \cite{M1}.

\begin{figure}[ht]
\centering
\includegraphics[scale=0.4]{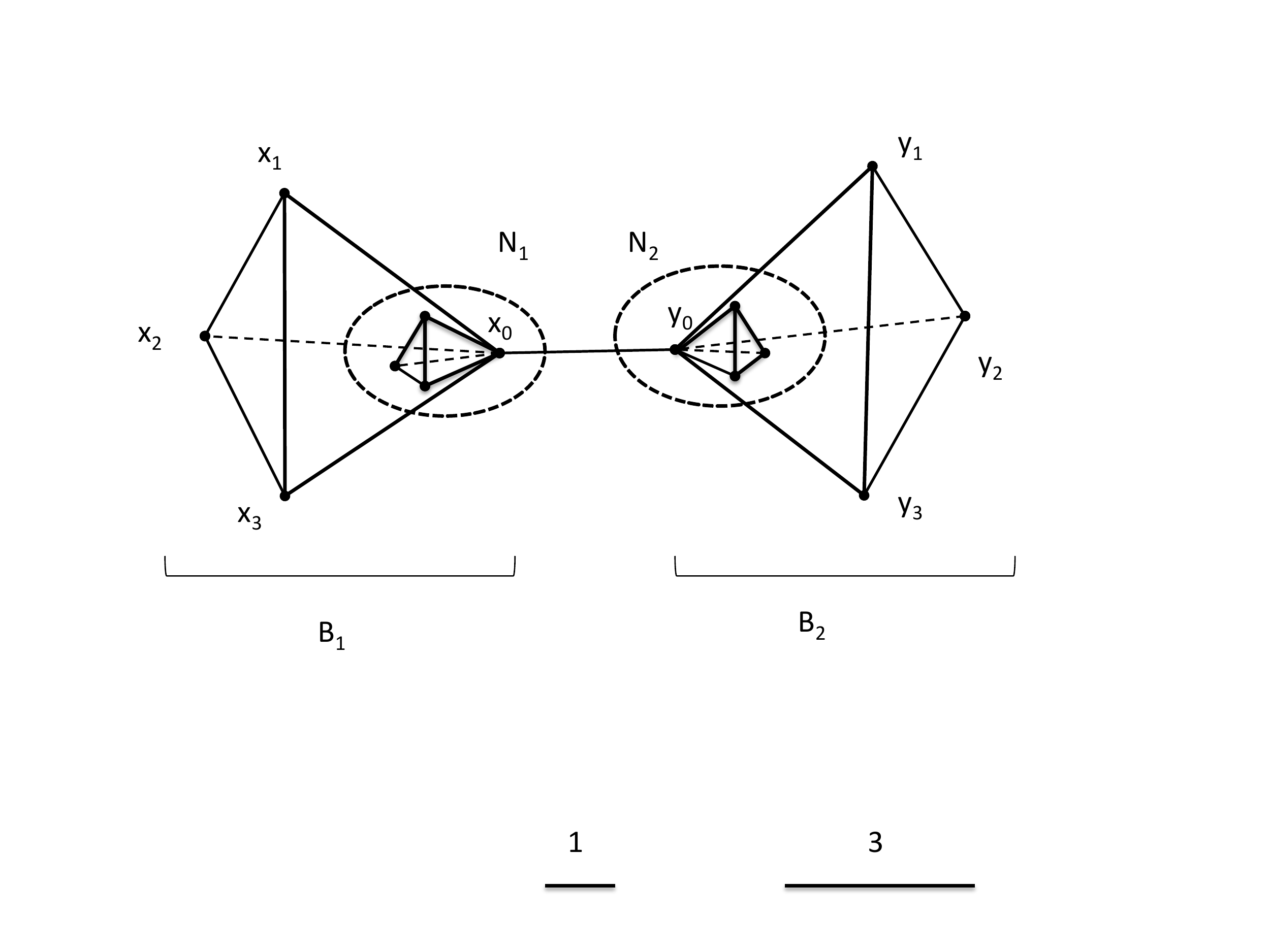}
\caption{$SL(\alpha)$ with $\alpha<3$ does not satisfy A2.}
\label{Ejp Chain_1b}
\end{figure}

\begin{ejp} Let $(X,d)$ be the graph from Figure \ref{Ejp Chain_1b}. 

Suppose the edges in $N_1,N_2$ have length 1 and the rest have length 3. The distances between vertices are measured as the minimal length of a path joining them. 


Let $Z:=\{x_0,y_0\}$ and $d'(x_0,y_0)=3$. Let $i\co (Z,d')\to (X,d)$ be the inclusion map. It is immediate to check that applying either $SL(\alpha)$ or $SL^*(\alpha)$ with $\alpha<3$ we obtain ultrametric spaces $(Z,u_Z)$, $(X,u_X)$ such that $u_Z(x_0,y_0)<u_X(x_0,y_0)$.
\end{ejp}

\begin{cor}\label{Cor: rich_2} $SL(\alpha)$ and $SL^*(\alpha)$ satisfy richness property.
\end{cor}



As we saw in \cite{M2}, $SL(\alpha)$ is semi-stable in the Gromov-Hausdorff sense. Unfortunately, most of the good stability properties of $SL$ do not hold. $SL(\alpha)$ and $SL^*(\alpha)$ are not stable in the Gromov-Hausdorff sense (see \cite{M2}) and it is not difficult to check that $SL^*(\alpha)$ is not semi-stable in the Gromov-Hausdorff sense. Small perturbations on the distances may affect the dimension of the Rips complex and to whether or not condition $ii$ applies. Also, they may affect the size of the components and yield very different graphs $G_\alpha^{t_i}$. Furthermore, changing the parameter $\alpha$ we may obtain a very different dendrogram. However, all the instability is produced by the unchaining conditions $ii$, $iii$  and $iv$. Thus, $\theta_\alpha$ and $\theta^*_\alpha$ may be compared with $\theta$ to, at least, keep track of the undesired effects on the stability introduced with the unchaining conditions. 

\begin{figure}[ht]
\centering
\includegraphics[scale=0.3]{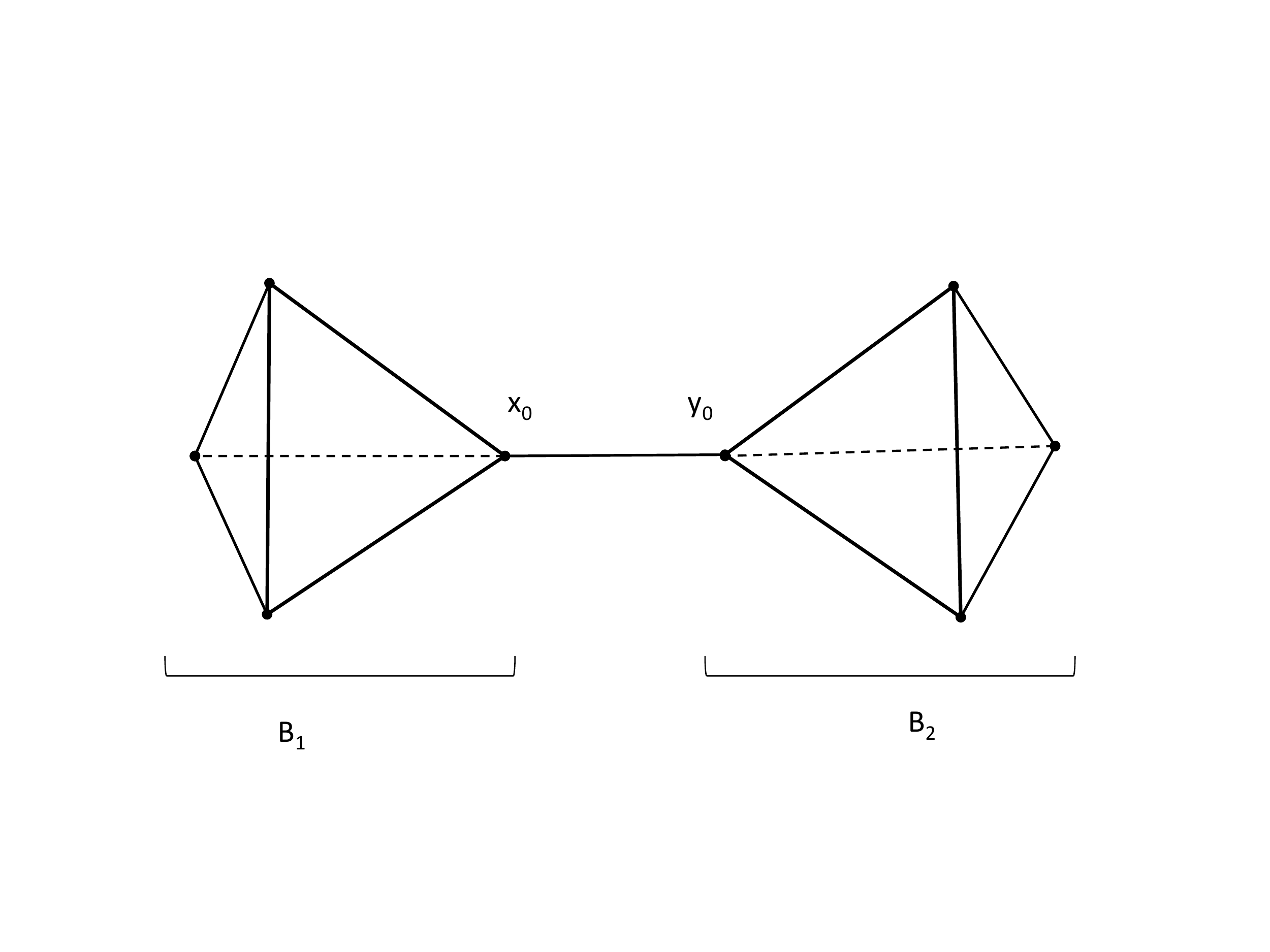}
\caption{A small perturbation in the distance between $x_0$ and $y_0$ produces very different dendrograms.}
\label{Ejp Chain_1}
\end{figure}

\begin{ejp}\label{Ejp: unstable}  Let $(X,d)$ be the graph from Figure \ref{Ejp Chain_1} where every edge has length 1 and let $(X,d')$ be the same graph where $d(x_0,y_0)=1+\varepsilon$ for some $\varepsilon>0$ and the rest of the edges have length 1. Let $\theta_1=\mathfrak{T}^{SL(1)}_\mathcal{D}(X,d)$ and $\theta'_1=\mathfrak{T}^{SL(1)}_\mathcal{D}(X,d')$.

As we saw above, $\theta_1(t)=\{\{x_0\},...,\{x_3\},\{y_0\},...\{y_3\}\}$ if $t<1$ and $\theta_1(1)=\{X\}$. Thus, if $\eta(\theta_1)=u$ if follows that $u(x,y)=1$ $\forall \, x,y\in X$. 

If we apply $SL(1)$ to $(X,d')$ we obtain that $\theta'_1(t)=\{\{x_0\},...,\{x_3\},\{y_0\},...\{y_3\}\}$ if $t<1$ and $\theta'_1(t)=\{B_1,B_2\}$ for $1\leq t<1+\varepsilon$. For $1+\varepsilon\leq t \leq 2$, by condition $ii$, there is no edge in $G^t_1$ between $B_1$ and $B_2$. Thus, $\theta'_1(t)=\{B_1,B_2\}$ for $1+\varepsilon \leq t < 2+\varepsilon$. For $t\geq 2+\varepsilon$, $\theta'_1(t)=X$. Thus, if  $\eta(\theta'_1)=u'$ if follows that $u'(x_i,x_j)=1$ $\forall \, x_i,x_j\in B_1$, $u'(y_i,y_j)=1$ $\forall \, y_i,y_j\in B_2$ and $u'(x_i,y_j)=2+\varepsilon$ $\forall \, (x_i,y_j)\in B_1\times B_2$.

In this case, $d_{\mathcal{GH}}((X,d),(X,d'))= \frac{\varepsilon}{2}$ and $d_{\mathcal{GH}}((X,u),(X,u'))= \frac{1+\varepsilon}{2}$. 
Therefore, $SL(\alpha)$ is not stable in the Gromov-Hausdorff sense.
\end{ejp}

Also, it is unstable under the change of the parameter $\alpha$.

\begin{ejp}  Let $(X,d')$ be the graph with the metric defined in Example \ref{Ejp: unstable}.

As we just saw, if  $\eta(\theta'_1)=u'$ if follows that $u'(x_i,x_j)=1$ $\forall \, x_i,x_j\in B_1$, $u'(y_i,y_j)=1$ $\forall \, y_i,y_j\in B_2$ and $u'(x_i,y_j)=2+\varepsilon$ $\forall \, (x_i,y_j)\in B_1\times B_2$.

If we apply $SL(3)$ to $(X,d')$ we obtain that $\theta'_3(t)=\{\{x_0\},...,\{x_3\},\{y_0\},...\{y_3\}\}$ if $t<1$ and $\theta'_3(t)=\{B_1,B_2\}$ for $1\leq t<1+\varepsilon$. For $1+\varepsilon\leq t \leq 2$, since $\alpha =3$ there is an edge in $G^t_3$ between $B_1$ and $B_2$. Thus, $\theta'_3(t)=\{X\}$ for $1+\varepsilon \leq t$. Hence, if  $\eta(\theta'_3)=u''$ if follows that $u''(x_i,x_j)=1$ $\forall \, x_i,x_j\in B_1$, $u''(y_i,y_j)=1$ $\forall \, y_i,y_j\in B_2$ and $u''(x_i,y_j)=1+\varepsilon$ $\forall \, (x_i,y_j)\in B_1\times B_2$.

Therefore, $d_{\mathcal{GH}}((X,u'),(X,u''))= \frac{1}{2}$. 
\end{ejp}

One may wonder if given $\alpha >\alpha'$ anything can be told about the corresponding dendrograms. In particular, given $\mathfrak{T}^{SL^*(\alpha)}_\mathcal{U}(X,d)=u^*_\alpha$ and $\mathfrak{T}^{SL^*(\alpha')}_\mathcal{U}(X,d)=u^*_{\alpha'}$, it is natural to ask if $u^*_{\alpha} \leq u^*_{\alpha'}$ or $u^*_{\alpha'}\leq u^*_{\alpha}$. This need not be true. In fact, it may fail  by conditions $iii$ and $iv$, see Example \ref{Ejp: unrefine_1}, or by condition $ii$, see Example \ref{Ejp: unrefine_2}.

\begin{figure}[ht]
\centering
\includegraphics[scale=0.4]{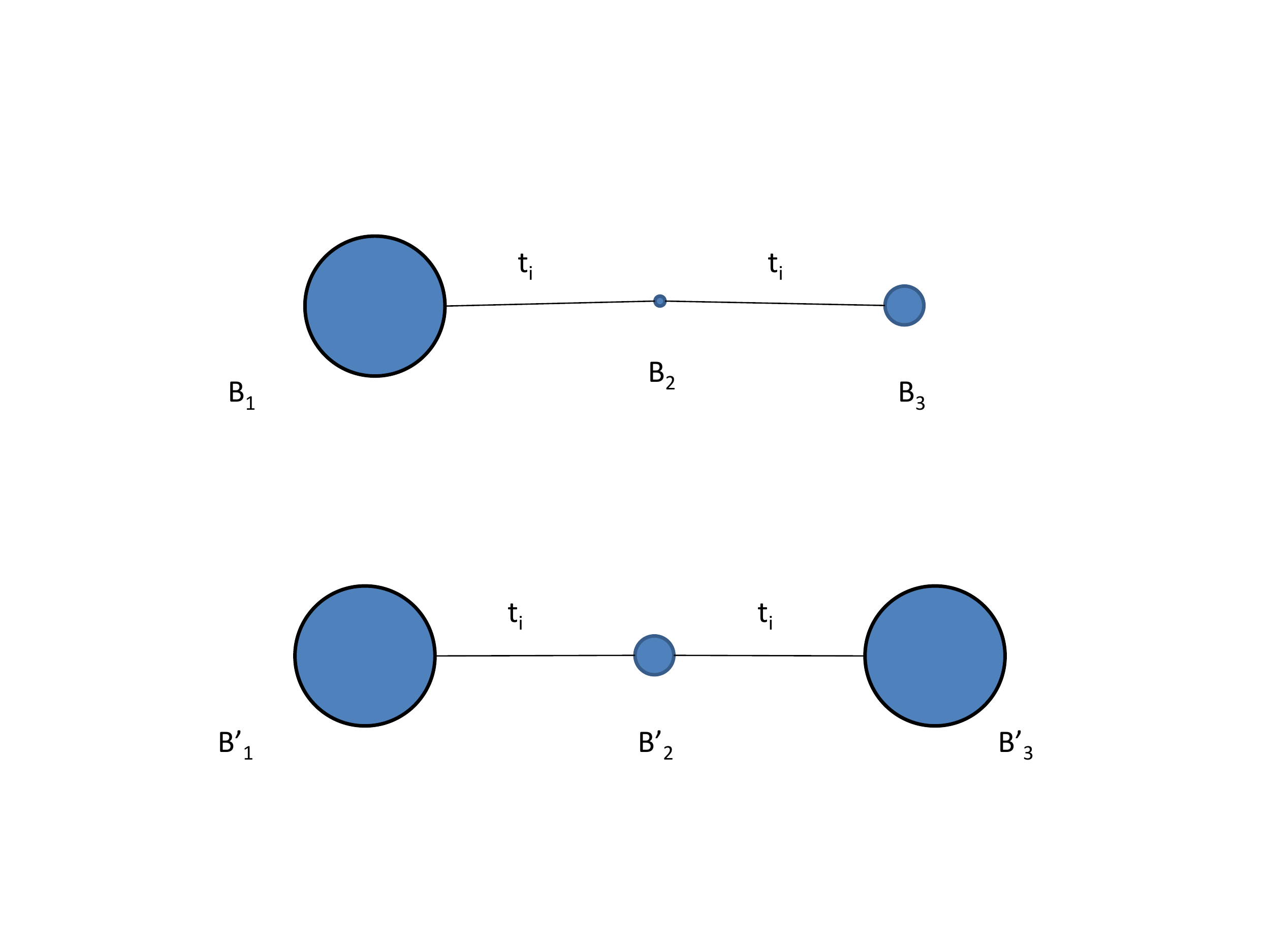}
\caption{A change on the parameter $\alpha$ may have different effects on the chaining through smaller blocks.}
\label{Unstable_1}
\end{figure}

\begin{ejp}\label{Ejp: unrefine_1} Let $\alpha >\alpha'$. Suppose that $\theta^*_\alpha(t_{i-1})=\theta^*_{\alpha'}(t_{i-1})=\{B_1,B_2,B_3\}$.  See the example above from Figure \ref{Unstable_1}. Now, suppose that conditions $i,$ $ii$ define edges $\{B_1,B_2\}$ and $\{B_2,B_3\}$ but not $\{B_1,B_3\}$ in both $G_\alpha^{t_i}$ and $G_{\alpha'}^{t_i}$.

Suppose that $\max_{1\leq l \leq 3}\{ \#(B_l)\}=\#(B_1)$. Also, let us suppose that $\alpha \cdot \#(B_2)< \#(B_1)$, $\alpha' \cdot \#(B_2)< \#(B_1)$, $\alpha \cdot \#(B_3)\geq \#(B_1)$ but $\alpha' \cdot \#(B_3)< \#(B_1)$. In this case, there is a unique connected component $A=\{B_1,B_2,B_3\}$ and $H_{\alpha'}(A)=\{B_1\}$ is connected while $H_{\alpha}(A)=\{B_1,B_3\}$ is not connected. Thus, $\theta^*_\alpha(t_i)=\{B_1,B_2,B_3\}$ and $\theta^*_{\alpha'}(t_i)=\{B_1\cup B_2\cup B_3\}=\{X\}$.

Suppose that $\theta^*_{\alpha'}(t_{i-1})=\theta^*_{\alpha'}(t_{i-1})=\{B'_1,B'_2,B'_3\}$.  See the example below from Figure \ref{Unstable_1}. Now, suppose that conditions $i,$ $ii$ define edges $\{B'_1,B'_2\}$ and $\{B'_2,B'_3\}$ but not $\{B'_1,B'_3\}$ in both $G_\alpha^{t_i}$ and $G_{\alpha'}^{t_i}$.

Suppose that $\max_{1\leq l \leq 3}\{ \#(B'_l)\}=\#(B'_1)$. Let us suppose that $\alpha \cdot \#(B'_3)> \#(B'_1)$, $\alpha' \cdot \#(B'_3)> \#(B'_1)$, $\alpha \cdot \#(B'_2)\geq \#(B'_1)$ but $\alpha' \cdot \#(B'_2)< \#(B'_1)$. In this case, there is a unique connected component $A=\{B_1,B_2,B_3\}$, $H_{\alpha}(A)$ has vertices $B'_1,B'_2,B'_3$ and it is connected while $H_{\alpha'}(A)=\{B'_1,B'_3\}$ is not connected. Thus, $\theta^*_\alpha(t_i)=\{B'_1\cup B'_2\cup B'_3\}=\{X\}$ and $\theta^*_{\alpha'}(t_i)=\{B'_1,B'_2,B'_3\}$.

Hence, even in the case when there is no chaining effect between adjacent blocks, $\theta^*_\alpha(t_i)$ need not refine $\theta^*_{\alpha'}(t_i)$ and $\theta^*_{\alpha'}(t_i)$ need not refine $\theta^*_{\alpha}(t_i)$. 

In particular, $u^*_{\alpha}\not \leq u^*_{\alpha'}$ and $u^*_{\alpha'}\not \leq u^*_{\alpha}$.
\end{ejp}

\begin{ejp}\label{Ejp: unrefine_2} Let $\alpha >\alpha'$. Suppose $\theta_\alpha(t_{i-1})=\theta_{\alpha'}(t_{i-1})=\{B_1,B_2,B_3,B_4\}$, $d(B_1,B_2)=d(B_3,B_4)=t_i$, $d(B_1,B_3)=t_{i+1}$ and the rest of respective distances between these blocks are bigger than $t_{i+1}$. See Figure \ref{Unstable_2}.


Since $\alpha >\alpha'$, we may assume, by condition $ii$, that there is an edge between $B_1,B_2$ and between $B_3,B_4$ in $G_\alpha^{t_i}$ but not in $G_{\alpha'}^{t_i}$. Thus, suppose $\theta_\alpha(t_i)=\{B_6,B_7\}$ while $\theta_{\alpha'}(t_i)=\{B_1,B_2,B_3,B_4\}$.

Now, we may assume that $dim(F_{t_{i+1}}(B_1)),dim(F_{t_{i+1}}(B_2))<dim(F_{t_{i+1}}(B_6))$ and $dim(F_{t_{i+1}}(B_3)),dim(F_{t_{i+1}}(B_4))<dim(F_{t_{i+1}}(B_7))$. Thus, we may also assume that, at $t_{i+1}$, for $\alpha'$ there is no edge between $B_6,B_7$ but for $\alpha$ there is an edge between $B_1,B_3$. Therefore, $\theta_{\alpha'}(t_{i+1})=\{B_6,B_7\}$ while $\theta_{\alpha}(t_{i+1})=\{B_5,B_2,B_4\}$. 

Hence,  $\theta_\alpha(t_i)$ does not refine $\theta_{\alpha'}(t_i)$ and $\theta_{\alpha'}(t_i)$ does not refine $\theta_{\alpha}(t_i)$.

In particular, it is immediate to check that $u_{\alpha}\not \leq u_{\alpha'}$ and $u_{\alpha'}\not \leq u_{\alpha}$.
\end{ejp}

\begin{figure}[ht]
\centering
\includegraphics[scale=0.4]{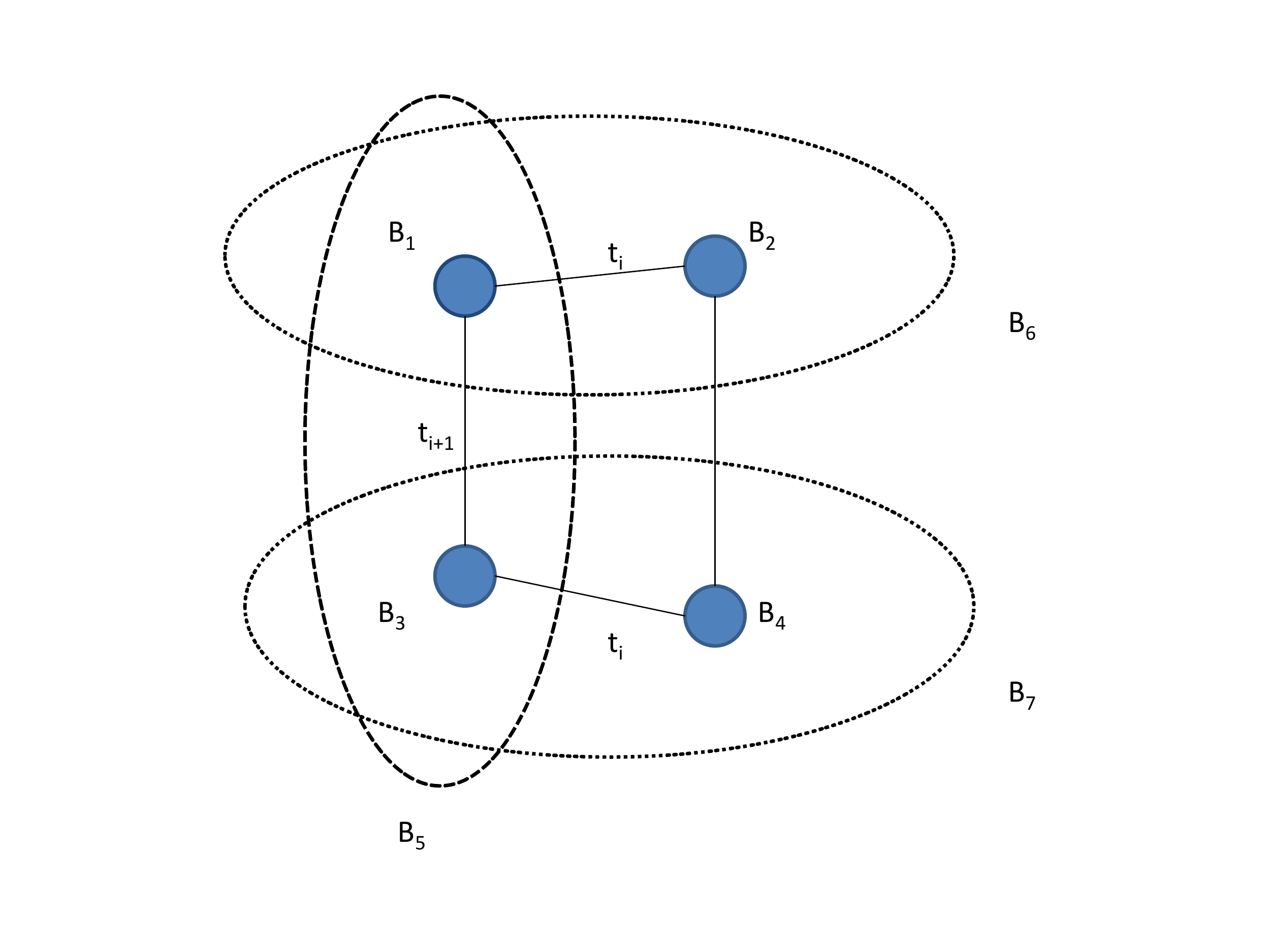}
\caption{A bigger $\alpha$ does not imply a smaller ultrametric.}
\label{Unstable_2}
\end{figure}

\begin{table}
\centering
\begin{tabular}{|l|c|c|c|c|c|} 
\hline
                           & $SL$   & $CL$    & $AL$  & $SL(\alpha)$  &  $SL^*(\alpha)$  \\
\hline
Permutation invariant      & \cmark & \cmark & \cmark & \cmark       &   \cmark  \\
\hline
Rich                       & \cmark & \cmark & \cmark & \cmark       & \cmark \\
\hline
$A1$                       & \cmark & \cmark & \cmark & \cmark       &   \cmark \\
\hline
$A2$                       & \cmark & \xmark & \xmark & \xmark       &   \xmark \\
\hline
$A3$                       & \cmark & \cmark & \cmark & \cmark       & \cmark \\
\hline
Semi-stable                & \cmark & \cmark & \cmark & \cmark &  \xmark \\
\hline
Stable                     & \cmark & \xmark & \xmark & \xmark &  \xmark  \\
\hline
Strongly chaining          & \cmark & \xmark & \xmark & \xmark   &  \xmark \\
\hline
Completely chaining        & \cmark & \xmark & \xmark & \xmark & \xmark \\
\hline
Weakly unchaining          & \xmark & \xmark & \xmark &  \cmark  &  \cmark  \\
\hline
$\alpha$-bridge-unchaining & \xmark & \xmark & \xmark & \xmark &  \cmark  \\
\hline

\end{tabular}
\caption{Overview of the properties satisfied by the hierarchical clustering methods discussed in this work.}
\label{tabla}
\end{table}

\section{Conclusions}\label{Section: Conclusions}

In the spirit of Kleinberg impossibility result we may consider $A1$ (the algorithm leaves ultrametric spaces invariant) and $A3$ (the distance between two points in the output is at least the minimal $\varepsilon>0$ so that there exists a $\varepsilon$-chain between them in the input) as basic desirable conditions for any $HC$ algorithm $\mathfrak{T}$. Thus, if we assume that 
$\mathfrak{T}$ satisfies $A1$ and $A3$, then either $\mathfrak{T}$ is exactly $SL$ or else, condition $A2$ (adding points to the input will never make increase the distance between previous points in the output) is not satisfied. In particular, 
condition $A2$ is not satisfied by the algorithms defined to treat the chaining effects: $SL(\alpha)$ and $SL^*(\alpha)$.

Apart from this inevitable difference, we prove that the properties $A1$, $A3$, permutation invariance and richness are satisfied by $SL(\alpha)$ and $SL^*(\alpha)$ and also by the classical linkage-based algorithms, $SL$, $CL$ and $AL$.

Their chaining and unchaining  properties were studied in \cite{M1}. The stability properties of linkage-based methods were analyzed in \cite{M2}. The main results are summarized in Table \ref{tabla}.



\end{document}